\documentclass[a4paper]{article}

\usepackage[english]{babel}
\usepackage[utf8x]{inputenc}
\usepackage{amsmath,amsthm,amsfonts}
\usepackage{float}
\usepackage{graphicx}
\usepackage[colorinlistoftodos]{todonotes}
\newtheorem{theorem}{Theorem}[section]
\newtheorem{lemma}{Lemma}
\newtheorem{remark}{Remark}[section]
\newcommand{\tabincell}[2]{\begin{tabular}{@{}#1@{}}#2\end{tabular}}

\usepackage{geometry}
\usepackage[ruled,vlined]{algorithm2e}

\title{Disease Prediction with a Maximum Entropy Method}
\author{Michael Shub\footnote{Math Department, City College and the Graduate Center of CUNY.}, ~Qing Xu\footnote{UniDT, Shanghai, China.}, ~Xiaohua(Michael) Xuan\footnote{UniDT, Shanghai, China.}~\footnote{This work was partially supported by the Smale Institute.}}
\date{}

\begin{document}
\maketitle

\begin{abstract}
	In this paper, we propose a maximum entropy method for predicting disease risks. It is based on a patient's medical history with diseases coded in ICD-10 which can be used in various cases. The complete algorithm with strict mathematical derivation is given. We also present experimental results on a medical dataset, demonstrating that our method performs well in predicting future disease risks and achieves an accuracy rate twice that of the traditional method. We also perform a comorbidity analysis to reveal the intrinsic relation of diseases.
\end{abstract}

\section{Introduction}

Disease prediction is an effective way to assess a person's health status. Studies have shown that in many cases, there are identifiable indicators or preventable risk factors before the onset of the patient's disease. These early warnings can effectively reduce the individual's risk of disease. Theoretically, this can reduce the number of treatments needed and increase the necessary effective interventions. However, the combination of problem factors caused by different diseases and the patient's past medical history are so complicated that no doctor can fully understand all of this. Currently, doctors can use family and health history and physical examinations to estimate the patient’s risk and guide laboratory tests to further evaluate the patient's health. However, these sporadic and qualitative "risk assessments" are usually only for a few diseases, depending on the experience, memory and time of the particular doctor. Therefore, the current medical care is after the fact. Once the symptoms of the disease appear, it is involved, rather than actively treating or eliminating the disease as soon as possible.

Today the prevailing model of prospective heath care is firmly based on the genome revolution. Indeed, technologies ranging from linkage equilibrium and candidate gene association studies to genome wide associations have provided an extensive list of disease-gene associations, offering us detailed information on mutations, SNPs, and the associated likelihood of developing specific disease phenotypes.

The basic assumption behind the research is that once we have classified all disease-related mutations, we can use various molecular biomarkers to predict each individual's susceptibility to future diseases, thus bringing us into a predictive medicine era. However, these rapid advances have also revealed the limitations of genome-based methods. Considering that the signals provided by most disease-related SNPs or mutations are very weak, it is becoming increasingly clear that the prospect of genome-based methods may not be realized soon.

Does this mean that prospective disease prediction methods must wait until genomics methods are sufficiently mature? Our purpose is to prove that the method based on medical history provides hope for the prospective prediction of disease.

In this paper, we mainly study the disease prediction and comorbidity of diseases. Our approach is distinctly different in that we are trying to build a general predictive system which can utilize a less constrained feature space, i.e. taking into account all available demographics and previous medical history. Moreover, we rely primarily on ICD-10-CM (International Classification of Diseases, Tenth Revision, Clinical Modification) codes (see Section 2) for making predictions to account for the previous medical history, rather than specialized test results.

\section{Data}

\subsection{Source Data and Population}

Our database comprises the medical records of 354,552 patients in China with a total of 2,904,257 hospital visits. The data was originally compiled from Insurance claims during 2007 to 2017. Such medical records are highly complete and accurate, and they are frequently used for epidemiological and demographic research.

The input for our methods consists of each patient's personal information, such as gender, birthday, treatment-date, and diagnosis history, provided per patient's visit. Each data record consists of a hospital visit, represented by a patient ID and a diagnosis code per visit, as defined by the International Classification of Diseases, Tenth Revision,Clinical Modification(ICD-10-CM). The International
Statistical Classification of Diseases and Related Health Problems (ICD) provides
codes to classify diseases and a wide variety of signs, symptoms, abnormal findings, social circumstances, and external causes of injury or disease. It is published by the
World Health Organization.

Each disease or health condition is given a unique code, and can be up to 6 characters long, such as A01.001. The first character is a letter while the others are digits. ICD-10 codes are hierarchical in nature, so the 6 characters codes can be collapsed to fewer characters identifying a small family of
related medical conditions. For instance, code A01.001 is a specific code for typhoid fever. This code can be collapsed to A01.

Moreover, we classify diseases of the same category into one class. For example, A90 is the code for Dengue fever (classical dengue) and A91 is the code for Dengue hemorrhagic fever. We classify them into the same class named F\_A90. Thus, the 20 thousand origin ICD-10 codes are classified into 429 classes.

A sample patient medical history is shown in Table 1. Each line represents one hospital visit. Demographic data are also available.


\begin{table}[H]
	\centering
	\caption{Medical History Sample}
	\begin{tabular}{cccc}
		\hline
		\multicolumn{1}{c}{\textbf{patient\_id}} & \multicolumn{1}{c}{\textbf{gender}} & \multicolumn{1}{c}{\textbf{treatment\_date}} & \multicolumn{1}{c}{\textbf{code}} \\ \hline
		14532                                     & F                                   & 2011-10-15                                   & F\_M47                            \\
		14532                                     & F                                   & 2011-11-19                                   & F\_N91                            \\
		14532                                     & F                                   & 2012-10-09                                   & F\_L20                            \\
		14532                                     & F                                   & 2012-10-19                                   & F\_N60                            \\
		14532                                     & F                                   & 2013-05-08                                   & F\_B37                            \\
		14532                                     & F                                   & 2013-06-04                                   & F\_H10                            \\
		14532                                     & F                                   & 2013-06-15                                   & F\_K04                            \\
		14532                                     & F                                   & 2013-08-23                                   & F\_L20 \\
		\hline                          
	\end{tabular}
\end{table}

In our medical database, the number of visits per patient ranges from 1 to 491, with a median of 4. Also, the average is 8.19. Table 2 shows the 20 most prevalent diseases in our database.

\begin{table}[H]
	\centering
	\caption{20 Most Prevalent Diseases}
	\begin{tabular}{cc}
		\hline
		\textbf{Disease}                                    & \textbf{Prevalence} \\
		\hline
		Acute upper respiratory infection                   & 20.88\%                  \\
		Hypertension and its complications                  & 7.35\%                   \\
		Dermatitis and pruritus                             & 3.75\%                   \\
		Gastritis and duodenitis                            & 3.49\%                   \\
		Chronic bronchitis                                  & 3.28\%                   \\
		Pulp, gum, and alveolar ridge diseases              & 3.15\%                   \\
		Hard tissue disease of teeth                        & 2.73\%                   \\
		Abnormal uterine and vaginal bleeding               & 2.42\%                   \\
		Chronic rhinitis, nasopharyngitis and   pharyngitis & 1.99\%                   \\
		Non-infectious gastroenteritis and   colitis        & 1.97\%                   \\
		Chronic ischemic heart disease                      & 1.93\%                   \\
		Inflammation of the vagina and vulva                & 1.72\%                   \\
		Pneumonia                                           & 1.63\%                   \\
		Abnormal thyroid (parathyroid) function             & 1.62\%                   \\
		Other diabetes                                      & 1.56\%                   \\
		Backache                                            & 1.54\%                   \\
		Acute lower respiratory infection                   & 1.51\%                   \\
		Cervical disc disease                               & 1.44\%                   \\
		Type II diabetes                                    & 1.37\%                   \\
		Female pelvic inflammatory disease                  & 1.18\%      \\ \hline            
	\end{tabular}
\end{table}

\subsection{Quantifying the Strength of Comorbidity Relationships}

In order to measure the correlation from disease comorbidity, we need to quantify the intensity of disease comorbidity by introducing the concept of distance between the two diseases. One difficulty of this method is that there are biases in different statistical measures, which overestimate or underestimate the relationship between rare or epidemic diseases. Given that the number of diagnoses (prevalence) for a particular disease follows a long tail distribution, these biases are important, which means that although most diseases are rarely diagnosed, a small number of diseases have been diagnosed in a large part of the population.

Therefore, quantifying comorbidity usually requires us to compare diseases that affect dozens of patients with diseases that affect millions of patients.

We will use two comorbidity measures to quantify the distance between two diseases: The Absolute Logarithmic Relative Risk (ALRR) and $\phi$-correlation($\phi$). 

The Emperical Relative Risk of observing a pair of diseases $i$ and $j$ affecting the same patient is given by
$$ERR_{ij}=\frac{C_{ij}N}{P_iP_j},$$
where $C_{ij}$ is the number of patients affected by both diseases, $N$ is the total number of patients in the population and $P_i$ and $P_j$ are the prevalences of diseases $i$ and $j$. 

$$ERR_{ij}=\frac{\frac{C_{ij}}{N}}{\frac{P_i}{N}\cdot\frac{P_j}{N}}.$$

Thus, the Relative Risk is defined as 
$$RR_{ij}=\frac{p_{ij}}{p_ip_j},$$

Here, $p_{ij}$ is the transition probability from disease $i$ to disease $j$ and $p_i$ is the incidence probability of disease $i$. The Absolute Logarithmic Relative Risk is defined as 
$$M_{ij}=\left|\log\left(\frac{p_{ij}}{p_ip_j}\right)\right|.$$

The emperical $\phi$-correlation, which is Pearson’s correlation for binary variables, can be expressed mathematically as
$$\frac{C_{ij}N-P_iP_j}{\sqrt{P_iP_j(N-P_i)(N-P_j)}}=\frac{\frac{C_{ij}}{N}-\frac{P_i}{N}\cdot\frac{P_j}{N}}{\sqrt{\frac{P_i}{N}\cdot\frac{P_j}{N}(1-\frac{P_i}{N})(1-\frac{P_j}{N})}}.$$

Therefore, the $\phi$-correlation is defined as 

$$\phi_{ij}=\frac{p_{ij}-p_ip_j}{\sqrt{p_ip_j(1-p_i)(1-p_j)}}.$$

These two comorbidity measures are not completely independent of each other, as they both increase with the number of patients affected by both diseases, yet both measures
have their intrinsic biases. For example, RR overestimates
relationships involving rare diseases and underestimates the
comorbidity between highly prevalent illnesses, whereas $\phi$ accurately discriminates comorbidities between pairs of diseases of similar prevalence but underestimates the comorbidity between rare and common diseases.

\section{Methodology}

In this section. We will formulate the maximum entropy method we used to predict disease risk.

\subsection{Some notations}\label{notation}
Suppose there are $n$ diseases and $N$ records. Let us use $d_i$ to denote disease $i$. A record is a pair of diseases $(d_i,d_j)$ which means that  there is a patient with a diagnosis of disease $d_j$ simultaneously or after disease $d_i$. Let us use $S_k$ to denote record $k$
($1\leq i\leq n, 1\leq k\leq N$).

Assume that $S_k=(f(k),g(k))$. Here, $f$ and $g$ are maps.
$$f,g:\{1,2,\cdots,N\}\rightarrow\{1,2,\cdots,n\}.$$
$f(k)$ is called the first disease and $g(k)$ is called the second disease in record $k$.

In this paper, we assume that $f$ and $g$ are surjective. If $f$ is not surjective, we can remove the diseases with indexes in $\{1,2,\cdots,n\}\backslash \text{Range}(f)$ from the medical history. Then the surjective assumption can be satisfied for $f$. The same can be done for $g$.

Assume that
\[X_{ik}=\left\{
\begin{split}
	&1,\ f(k)=i\\
	&0,\ f(k)\neq i
\end{split}
\right.,\quad 
Y_{kj}=\left\{
\begin{split}
	&1,\ g(k)=j\\
	&0,\ g(k)\neq j
\end{split}
\right..\]

Denote by 
\[\chi_C=\left\{
\begin{aligned}
	&1,\ x\in C\\
	&0,\ x \notin C
\end{aligned}
\right..\]
Let $A=XY,B=YX$. Then 

$$A_{ij}=\sum_{k=1}^NX_{ik}Y_{kj}=\sum_{k=1}^N\chi_{\{f(k)=i,g(k)=j\}}$$
is the number of patients who suffer from disease $i$ before disease $j$.

$$B_{km}=\sum_{j=1}^nY_{kj}X_{jm}=\sum_{j=1}^n\chi_{\{f(m)=j,g(k)=j\}}.$$

$B$ is a matrix with entries ranged in $\{0,1\}$. $B_{km}=1$ if and only if there is a disease $j$ such that the patient in record $k$ suffer a second disease $j$ and the patient in record $m$ suffer a first disease $j$. Our task is to evaluate the transition probability from record $k$ to record $m$.

\subsection{Entropy for Markov Chains}

Suppose $M$ is a non-negative matrix. If for each $k,m\in\{1,2,\cdots,N\}$, there exists $l\geq 1$ such that $(M^l)_{k,m}>0$, then $M$ is said to be irreducible.

Now we define the entropy for Markov chains. A matrix $B$ is called a skeleton matrix if its entries are either $0$ or $1$. A non-negative matrix $P$ as called a Markov transition matrix if
$$\sum_{m=1}^{N}P_{k,m}=1,\quad \forall \ k=1,2,\cdots,N.$$

Moreover, if $P_{k,m}>0 \Leftrightarrow B_{k,m}=1$, then $P$ is called the Markov transition matrix associated with the skeleton matrix $B$.

For a non-negative vector $p=(p_1,\cdots,p_N)$, if 
$$\sum_{k=1}^{N}p_k=1,\quad pP=p,$$
then $p$ is called a stationary distribution of $P$.

For a non-negative matrix $W=(w_{k,m})$, if
$$\sum_{k,m=1}^Nw_{k,m}=1,$$
and for all $1\leq k\leq N$,
$$\sum_{m=1}^{N}w_{m,k}=\sum_{m=1}^{N}w_{k,m}.$$
Then $W$ is called a Markov weight matrix.

Here are some connections between Markov transition matrix and Markov weight matrix.

For a Markov transition matrix $P$ with stationary distribution $p$. Define
$$w_{k,m} = p_kP_{k,m}.$$
Then it is easy to verify that $W=(w_{k,m})$ is a Markov weight matrix.

On the contrary, given a Markov weight matrix $W=(w_{k,m})$, set
$$p_k=\sum_{m=1}^Nw_{k,m},\quad P_{k,m}=\frac{w_{k,m}}{p_k}.$$
Then $W$ is the Markov weight matrix associated with $P$.

Now we define the entropy for a Markov transitin matrix. First, let us consider the chain with length $l$.

\begin{align*}
	H_l(P) &= -\sum_{i_0=1}^N\sum_{i_1=1}^N\cdots\sum_{i_l=1}^Np_{i_0}P_{i_0,i_1}\cdots P_{i_{l-1},i_l}\log(p_{i_0}P_{i_0,i_1}\cdots P_{i_{l-1},i_l})\\
	&=-\sum_{i_0=1}^Np_{i_0}\log(p_{i_0})-l\sum_{k=1}^{N}\sum_{m=1}^{N}p_k P_{k,m}\log(P_{k,m}).
\end{align*}

So the entropy for a chain with unit length is defined as

$$H(P)=\lim_{l\rightarrow\infty}\frac{H_l(P)}{l}=-\sum_{k=1}^{N}\sum_{m=1}^{N}p_k P_{k,m}\log(P_{k,m}).$$

\subsection{Maximum Entropy Theorem}
The principle of maximum entropy is a basic principle in information theory(see e.g. \cite{Shan}). It states that the probability distribution which best represents the current state of knowledge is the one with largest entropy. Since the distribution with the maximum entropy is the one that makes the fewest assumptions about the true distribution of data, the principle of maximum entropy can be seen as an application of Occam's razor(see e.g. \cite{ocam}).

\begin{theorem}\label{thm1}
Suppose $B$ is irreducible. $\lambda$ is the maximum eigenvalue of $B$, and $l=(l_1,\cdots,l_N),r=(r_1,\cdots,r_N)^T$ are the corresponding left and right eigenvectors with
$$\sum_{k=1}^Nl_kr_k=1.$$
Then the entropy of the Markov chain associated with the skeleton matrix $B$ attains the maximum $\log\lambda $ when
$$w_{k,m}=\frac{1}{\lambda}B_{k,m}l_kr_m,\quad 1\leq k,m\leq N.$$
Here, $W=(w_{k,m})$ is the weight matrix for $B$.
\end{theorem}
\begin{proof}
	
	see Appendix.
	
\end{proof}

\begin{theorem}\label{thm2}
	Suppose $A$ is irreducible. $\lambda$ is the maximum eigenvalue of $A$, and $L=(L_1,\cdots,L_n),R=(R_1,\cdots,R_n)^T$ are the corresponding left and right eigenvectors with
	$$\sum_{j=1}^nL_jR_j=1.$$
	Then the entropy of the Markov chain associated with the skeleton matrix $B$ attains the maximum $\log\lambda $ when
	$$v_{ij}=\frac{1}{\lambda}A_{ij}L_iR_j,\quad 1\leq i,j\leq n.$$
	Here, $V=(v_{ij})$ is the weight matrix for $A$.
\end{theorem}
\begin{proof}
	
	see Appendix.
	
\end{proof}

\begin{remark}
	Recall that we have assumed that $f$ and $g$ are surjective. Therefore, $X$ and $Y$ have rank $n$. Since $A=XY,B=YX$, we have that
	
	$$\det (\lambda I_N-YX) = \lambda ^{N-n}\det(\lambda I_n -XY).$$
	
	Therefore, the eigenvalues of $A$ and $B$ are the same except for the zeros. In particular, the largest eigenvalue of $A$ and $B$ are the same.
	
	Moreover, $YA=BY$. Suppose $a$ is an eigenvector of $A$ with eigenvalue $\mu$, then 
	$$\mu Ya = YAa= BYa.$$
	Since $a\neq 0$ and $Y$ is injective as a map from $\mathbb{R}^n$ to  $\mathbb{R}^N$, $Ya\neq 0$. Hence, $Ya$ is the eigenvector of $B$ with eigenvalue $\mu$.
\end{remark}

\subsection{Algorithm for Probability Estimation}

Following is the algorithm for estimating the related probability.

\begin{algorithm}[H]\label{alg1}
	\SetAlgoLined
	\textbf{Step 1.} Compute the matrix $A$.
	
	\textbf{Step 2.} Use power method to compute the maximum eigenvalue $\lambda$ of $A$ with the corresponding left and right eigenvectors $L_0$ and $R_0$. Let 
	$$L = \frac{L_0}{\sqrt{L_0\cdot R_0}},\quad R = \frac{R_0}{\sqrt{L_0\cdot R_0}}.$$
	
	\textbf{Step 3.} Compute the weight matrix $V$ as follows.
	$$v_{ij}=\frac{1}{\lambda}A_{ij}L_iR_j,\quad 1\leq i,j\leq n.$$
	
	\textbf{Step 4.} Compute the transition probability as follows.
	$$p_{ij}=\frac{v_{ij}}{\sum\limits_{l=1}^nv_{il}}.$$
	
	\textbf{Step 5.} Compute the stationary distribution as follows.
	$$p_i = L_iR_i.$$

	\caption{Probability Estimation}
\end{algorithm}

\subsection{Method for Disease Prediction}\label{mdp}
The prediction task is to predict the diseases that a person is most likely to have if we know that he has already suffered from diseases $(d_{i_1},d_{i_2},\cdots,d_{i_T})$ which are ordered by occurance.

We first calculate the probability by Algorithm \ref{alg1}. Then we construct the following quantity
$$r_j=\frac{1}{T}\sum_{l=1}^{T}p_{i_l,j}.\quad (1\leq j\leq n)$$

Then we sort the $r_j$ and choose the top 5 disease as the predicted diseases for a person.

We also make an additional assumption, that is, the latest disease take a highest weight. Thus, some modifications are made. We first construct a decreasing sequence $\{a_n\}_{n\geq 1}$ such that $a_n>0$. (For example, $a_n=1/n^2$). Then we modify $r_j$ as
$$r_j=\frac{\sum\limits_{l=1}^{T}a_{T+1-l}\cdot p_{i_l,j}}{\sum\limits_{l=1}^Ta_l}.$$

And we use the modified $r_j$ to choose the top 5 disease as the predicted diseases for a person.

\section{Experiments}
\subsection{Data Cleaning}
The diseases are classified by F-code as described in section 2, and there are 429 F-coded diseases in total. If someone suffered from the same disease for many times, then we keep the earliest record and remove the others. For example, for the patient with patientid=123770, she suffered from mucopurulent conjunctivitis on 2015-12-09 and 2016-03-08, then the record with 2016-03-08 is removed from the history.

We clean the data and collect the records of the same patient together into one record. The history column recorded the patient's disease history and the diseases are sorted by time and separated by a comma.

The following table is a sample of the cleaned data.
\begin{table}[H]
	\centering
	\caption{Sample of Cleaned Data}
	\begin{tabular}{cc}
		\hline
		\multicolumn{1}{c}{\textbf{patient\_id}} & \multicolumn{1}{c}{\textbf{history}}                                  \\ \hline
		123770                                    & F\_H10,F\_H00                                                         \\
		135086                                    & F\_M65,F\_J00,F\_K29,F\_K01,F\_K04                                    \\
		400195                                    & F\_J00,F\_K29                                                         \\
		3218331                                   & F\_J00,F\_J40                                                         \\
		119151                                    & F\_J00,F\_N60                                                         \\
		102519                                    & F\_J00,F\_L50,F\_E34,F\_E01                                           \\
		1503387                                   & F\_K29,F\_K01,F\_I83                                                  \\
		7044682                                   & F\_J00,F\_J20,F\_J40                                                  \\
		182660                                    & F\_E01,F\_J00,F\_J20                                                  \\
		1888934                                   & F\_K31,F\_K29,F\_K22,F\_K50,F\_J00,F\_J40,F\_J20,F\_M70,F\_J30,F\_J34\\ \hline
	\end{tabular}
\end{table}

\subsection{Calculate the Probability}
We construct the matrix $A$ as follows. 

First we initialize a $429\times 429$ matrix with entries equal to $0$. We also contruct a map $\varphi:\text{F-codes}\rightarrow \{0,1,2,\cdots,428\}$ to index the diseases.
Next, for history $(h_1,h_2,\cdots,h_T)$, we set
$$A_{\varphi(h_i),\varphi(h_j)}\leftarrow A_{\varphi(h_i),\varphi(h_j)}+1,\quad 1\leq i\leq j \leq T.$$
Thus, we establish the matrix $A$.

Next, we use the power method to calculate the maximum eigenvalue $\lambda $ of $A$ and the corresponding left and right eigenvectors $L$ and $R$.

After that, we derive the Markov weight matrix and the transition probability as described in Algorithm \ref{alg1}.

Finally, we can calculate $r_j$ as described in subsection \ref{mdp} and derive the related disease prediction.

\subsection{Results of Accuracy}
To compare the result, We use a method used previously by the insurance company as the benchmark. This method is called the emperical methods, that is, to calculate the incidence rate of diseases and use the top 5 prevalent diseases as predicition for each person.

We use 300,000 people's records to calculate $p_{ij}$ and also the top 5 diseases. For another 10,000 people, we use their records from 2007-2014 to calculate the diseases with the highest $r_j$ which is described in the previous section and choose the 5 diseases with highest $r_j$-score as prediction, which is known as the maximum entropy method. 

Then we examine the diseases they suffer from during 2015-2017 to see how many diseases is accurately predicted by these two methods. 

The measurement we use is called the hit rate. It is defined as follows.
$$H=\frac{|A\cap B|}{|B|},$$
where A is the disease set predicted by the model and B is the disease set that a person suffer from during 2015-2017.

If we predict 5 diseases using the maximum entropy method, the hit rate is 31.89\%.
As a contrast, the hit rate is 16.55\% for the empirical method.

We also compare the hit rate with 1/2/3/4 predictions for the two methods. The following table summarize the result.

\begin{table}[H]
	\centering
	\caption{Comparisons of Hit Rate}
	\begin{tabular}{|c|c|c|}
		\hline
		number of predictions & maximum entropy method & empirical method \\ \hline
		1                     & 15.01\%                & 7.54\%           \\ \hline
		2                     & 20.50\%                & 10.67\%          \\ \hline
		3                     & 25.21\%                & 13.55\%          \\ \hline
		4                     & 29.01\%                & 14.92\%          \\ \hline
		5                     & 31.89\%                & 16.55\%          \\ \hline
	\end{tabular}
\end{table}

We can see from the table that the hit rate of the maximum entropy method is approximately twice that of the empirical method.

\subsection{Comorbidity Analysis}
We first study the ALRR. Recall that $M$ is calculated as follows.
$$M_{ij}=\left|\log\left(\frac{p_{ij}}{p_ip_j}\right)\right|.$$

If disease $i$ and disease $j$ are independent, then
$M_{ij}$ is close to $0$. So if $M_{ij}$ is large, then disease $i$ and disease $j$ are highly correlated.

If disease $i$ is high blood pressure and disease $j$ is type II diabetes. Then
$$M_{ij}=2.17,\quad M_{ji}=2.39.$$

Here is list of diseases with high LRR.

\begin{table}[H]
	\caption{ALRR of maximum entropy method}
	\centering
	\begin{tabular}{|c|c|c|}
		\hline
		disease $i$                               & disease $j$                                          & $M_{ij}$      \\ \hline
		Type II diabetes                        & Type I diabetes                                    & 3.40 \\ \hline
		Pulmonary heart disease                 & Acute ischemic heart disease                       & 3.35 \\ \hline
		Type II diabetes                        & atherosclerosis                                    & 3.18 \\ \hline
		Diseases of lip, tongue and oral mucosa & \tabincell{c}{Malignant tumors of the lip,\\ mouth and pharynx}     & 2.98 \\ \hline
		heart failure                           & Arrhythmia                                         & 2.96 \\ \hline
		Metabolic disorders                     & renal failure                                      & 2.84 \\ \hline
		emphysema                               & asthma                                             & 2.80 \\ \hline
		pneumonia                               & Bronchiectasis                                     & 2.67 \\ \hline
		high blood pressure                     & Type II diabetes                                   & 2.47 \\ \hline
		high blood pressure                     & renal failure                                      & 2.46 \\ \hline
		alopecia                                & Seborrheic keratosis                               & 2.41 \\ \hline
		heart failure                           & Anal and rectal disorders                          & 2.41 \\ \hline
		Type II diabetes                        & high blood pressure                                & 2.39 \\ \hline
		heart failure                           & Peptic ulcer                                       & 2.38 \\ \hline
		high blood pressure                     & atherosclerosis                                    & 2.36 \\ \hline
		Alzheimer disease                       & Sleep disorders                                    & 2.33 \\ \hline
		high blood pressure                     & \tabincell{c}{Cerebral hemorrhage or infarction\\ and its sequelae} & 2.33 \\ \hline
		Over nutrition                          & Other diabetes                                     & 2.27 \\ \hline
		Pulmonary heart disease                 & Arrhythmia                                         & 2.22 \\ \hline
		high blood pressure                     & heart failure                                      & 2.21 \\ \hline
		Pituitary hyperfunction                 & Joint disorder                                     & 2.12 \\ \hline
		Pulmonary heart disease                 & arthritis                                          & 2.08 \\ \hline
	\end{tabular}
\end{table}

The following table is a list of diseases such that $M_{ij}$ differs from $M_{ji}$ .

\begin{table}[H]
	\centering
	\caption{Asymmetric ALRR of maximum entropy method}
	\begin{tabular}{|c|c|c|c|}
		\hline
		disease $i$                                           & disease $j$                                                    & $M_{ij}$      & $M_{ji}$      \\ \hline
		Female pelvic inflammatory disease                  & Trichomoniasis                                               & 1.13 & 3.40 \\ \hline
		nephritic nephrotic syndrome                        & heart failure                                                & 1.51  & 3.35 \\ \hline
		Metabolic disorders                                 & Malignant tumor of skin & 1.37 & 3.19 \\ \hline
		Esophageal diseases                                 & Splenic diseases                                             & 3.40 & 1.58 \\ \hline
		anemia                                              & hypotension                                                  & 2.89 & 1.30 \\ \hline
		Other diabetes                                      & Central nervous system diseases & 1.26  & 2.84\\ \hline
		Benign tumor of uterus                              & \tabincell{c}{Tumors with undetermined or unknown\\ endocrine gland dynamics} & 2.89 & 1.31 \\ \hline
		Arrhythmia                                          & Mental and behavioral disorders  & 2.58 & 1.03 \\ \hline
		Arthrosis                                           & epilepsy                                                     & 0.71 & 1.89 \\ \hline
		Arrhythmia                                          & Diseases of autonomic nervous system                         & 2.63 & 1.47 \\ \hline
		Headache syndrome                                   & Other diseases of arteries and arterioles                    & 1.58 & 2.74 \\ \hline
		\tabincell{c}{Acute pancreatitis and other\\ diseases of pancreas} & Type II diabetes                                             & 2.78 & 1.70 \\ \hline
		asthma                                              & emphysema                                                    & 1.73 & 2.80 \\ \hline
		Ankylosis and other spondylosis                     & Hypopituitarism                                              & 1.87 & 0.80 \\ \hline
		Arthrosis                                           & Myasthenia and primary muscle diseases                       & 2.71 & 1.65  \\ \hline
		Malignant tumors of digestive organs                & Hemangioma and lymphangioma                                  & 3.40 & 2.33 \\ \hline
		Refractive and accommodative disorders              & glaucoma                                                     & 2.41 & 1.35 \\ \hline
		Other diabetes                                      & Optic neuropathy                                             & 1.71 & 2.74 \\ \hline
		Chronic ischemic heart disease                      & Pericardial disease  & 1.26 & 2.28 \\ \hline
		Type II diabetes                                    & Over nutrition  & 0.82 & 1.83 \\ \hline
	\end{tabular}
\end{table}

Next, we consider the $\phi$-correlation. Recall that 
$$\phi_{ij}=\frac{p_{ij}-p_ip_j}{\sqrt{p_ip_j(1-p_i)(1-p_j)}}.$$

Next table display 20 disease pairs with high $\phi$-correlation.

\begin{table}[H]
	\caption{$\phi$-correlation of maximum entropy method} 
	\centering
	\begin{tabular}{|c|c|c|}
		\hline
		\textbf{disease $i$}                                                                & \textbf{disease $j$}                                  & \textbf{$\phi$} \\ \hline
		\tabincell{c}{Mania, bipolar, depression, and\\ anxiety disorders }                                & sleep disorder                                   & 68.23        \\ \hline
		Type II diabetes                                                                    & Hypertension and its complications                  & 67.75        \\ \hline
		Headache syndrome                                                                   & \tabincell{c}{Pulp, gums and edentulous alveolar\\ ridge diseases}   & 60.72  \\       \hline
		Arrhythmia                                                                          & Hypertension and its complications                  & 58.55        \\ \hline
		Muscle disorders                                                                    & Backache                                            & 57.28        \\ \hline
		Shingles                                                                            & Dermatitis and pruritus                             & 55.30        \\ \hline
		Headache syndrome                                                                   & Backache                                            & 53.33        \\ \hline
		Benign uterine tumor                                                                & Abnormal uterine and vaginal bleeding               & 46.75        \\ \hline
		\tabincell{c}{Upper respiratory tract diseases such as chronic\\ laryngitis and laryngotracheitis} & \tabincell{c}{Chronic rhinitis, nasopharyngitis\\ and pharyngitis}   & 42.54        \\ \hline
		Other disorders of kidney and ureter                                                & Other disorders of the urinary system               & 42.30        \\ \hline
		anemia                                                                              & \tabincell{c}{Pulp, gums and edentulous alveolar\\ ridge diseases}   & 40.85        \\ \hline
		cellulitis                                                                          &\tabincell{c}{ Dermatophytes and other superficial\\ fungal diseases} & 40.74        \\ \hline
		Other disorders of male reproductive   organs                                       & Prostatic hyperplasia and prostatitis               & 40.69        \\ \hline
		Urethral disorders                                                                  & Other disorders of the urinary system               & 39.66        \\ \hline
		Other disorders of bone                                                             & Osteoporosis without pathological fracture          & 34.96        \\ \hline
		\tabincell{c}{Upper respiratory tract diseases such as chronic\\ laryngitis and laryngotracheitis} & Chronic bronchitis                                  & 33.39        \\ \hline
		Other diseases of the digestive system                                              & Gastritis and duodenitis                            & 29.68        \\ \hline
		Type II diabetes                                                                    & Metabolic disorders                                 & 27.51        \\ \hline
		Type II diabetes                                                                    & Dermatitis and pruritus                             & 27.16        \\ \hline
		Arrhythmia                                                                          & sleep disorder                                      & 27.11        \\ \hline
	\end{tabular}
\end{table}

We can see from table 5 many disease pairs with large $M_{ij}$, such as type II diabetes and hypertension and its complications, which imply that such diseases have intrinsic relations. Tedesco \cite{Ted} have mentioned that Hypertension is frequently associated with diabetes mellitus and its prevalence doubles in diabetics compared to the general population. This high prevalence is associated with increased stiffness of large arteries. Our result is consistent with their medical research.

\section{Conclusions}
In this paper, we propose a maximum entropy method for predicting disease risks. It is based on a patient's medical history with diseases coded in ICD-10 which can be used in various cases. The complete algorithm with strict mathematical derivation is given. We also present experimental results on a medical dataset, demonstrating that our method performs well in predicting future disease risks and achieves an accuracy rate twice that of the traditional method. We also perform a comorbidity analysis to reveal the intrinsic relation of diseases.

\section*{Acknowledgement}

We would thank Franco Mueller, Jonathan Brezin and Matt Grayson for their collaboration on an early version of this research.

\section*{Appendix}
\begin{proof}[Proof of Theorem \ref{thm1}]
	Suppose $W=(w_{k,m})$ is the weight matrix. Then the  entropy of the Markov chain can be rewritten as
	\begin{align*}
		H(w)&=-\sum_{k=1}^N\sum_{m=1}^Nw_{k,m}\left(\log(w_{k,m})-\log\left(\sum_{m=1}^Nw_{k,m}\right)\right)\\
		&=-\sum_{k=1}^N\sum_{m=1}^Nw_{k,m}\log(w_{k,m})+\sum_{k=1}^N\left(\sum_{m=1}^Nw_{k,m}\right)\log\left(\sum_{m=1}^Nw_{k,m}\right).
	\end{align*}
	Let us construct the Lagrangian
	\begin{align*}
		L&=-\sum_{k=1}^N\sum_{m=1}^Nw_{k,m}\log(w_{k,m})+\sum_{k=1}^N\left(\sum_{m=1}^Nw_{k,m}\right)\log\left(\sum_{m=1}^Nw_{k,m}\right)\\
		&\qquad\qquad +\sum_{k=1}^Nh_k\left(\sum_{m=1}^Nw_{m,k}-\sum_{m=1}^Nw_{k,m}\right)+\mu\left(1-\sum_{k=1}^N\sum_{m=1}^Nw_{k,m}\right).
	\end{align*}
	If $w_{k,m}\neq0$, we have that
	$$\frac{\partial L}{\partial w_{k,m}}=-1-\log(w_{k,m})+\log\left(\sum_{m=1}^Nw_{k,m}\right)+1-h_k+h_m-\mu=0.$$
	$$\frac{w_{k,m}}{\sum\limits_{m=1}^Nw_{k,m}}=e^{-\mu}\frac{e^{h_m}}{e^{h_k}}.$$
	If $w_{k,m}=0$, then $B_{k,m}=0$. Therefore,
	$$\frac{w_{k,m}}{\sum\limits_{m=1}^Nw_{k,m}}=e^{-\mu}\frac{B_{k,m}e^{h_m}}{e^{h_k}}.$$
	Set $\lambda=e^\mu$, then
	$$\sum_{m=1}^NB_{k,m}e^{h_m}=\lambda e^{h_k}.$$
	By the Perron-Frobenius theorem(see e.g. \cite{CDM}), There are no nonnegative eigenvectors for $B$ other than the
	Perron vector $r$ and its positive multiples. Hence, 
	$$\frac{e^{h_m}}{e^{h_k}}=\frac{r_m}{r_k}.$$
	And
	$$P_{k,m}=\frac{w_{k,m}}{\sum\limits_{m=1}^Nw_{k,m}}=\frac{1}{\lambda}B_{k,m}\frac{r_m}{r_k}.$$
	On the other hand,
	$$\lambda\frac{w_{k,m}}{r_m}=B_{k,m}\frac{p_k}{r_k}\Rightarrow\lambda\frac{p_m}{r_m}=\sum_{k=1}^NB_{k,m}\frac{p_k}{r_k}.$$
	Therefore, there exists $t>0$ such that
	$$\frac{p_k}{r_k}=tl_k.$$
	$$\sum_{k=1}^Nl_kr_k=1\Rightarrow t=1.$$
	Hence, 
	$$p_k=l_kr_k,\quad w_{k,m}=\frac{1}{\lambda}l_kB_{k,m}r_m.$$
	Recall that
	$$H=-\sum_{k=1}^{N}\sum_{m=1}^{N}p_k P_{k,m}\log(P_{k,m}),$$
	and $0\cdot \log 0 =\lim\limits_{x\rightarrow 0+}x\log x =0$. It follows that
	\begin{align*}
		H&=-\sum_{k=1}^{N}\sum_{m=1}^{N}l_kr_k\frac{1}{\lambda}B_{k,m}\frac{r_m}{r_k}\log\left(\frac{1}{\lambda}B_{k,m}\frac{r_m}{r_k}\right)\\
		&=-\sum_{k=1}^{N}\sum_{m=1}^{N}l_k\frac{1}{\lambda}B_{k,m}r_m(\log B_{k,m} +\log r_m - \log r_k -\log\lambda )\\
	\end{align*}
	Since $B_{k,m}\in\{0,1\}$, $B_{k,m}\log B_{k,m} =0$,
	\begin{align*}
		H&=-\sum_{k=1}^{N}\sum_{m=1}^{N}l_k\frac{1}{\lambda}B_{k,m}r_m(\log r_m - \log r_k)+\log\lambda \\
		&=-\sum_{m=1}^{N}l_mr_m\log r_m + \sum_{k=1}^{N}l_kr_k\log r_k+\log\lambda \\
		&= \log\lambda.
	\end{align*}
\end{proof}
\bigskip

Next, we will prove Theorem \ref{thm2}. We first prove an auxiliary lemma.

\begin{lemma}
	Suppose $\lambda$ is the maximum eigenvalue of $B$ and $X,Y$ are matrices in section \ref{notation}. $l=(l_1,\cdots,l_N),r=(r_1,\cdots,r_N)^T$ are the corresponding left and right eigenvectors with
	$$\sum_{k=1}^Nl_kr_k=1.$$
	Then
	\begin{equation}\label{eq}
		\begin{pmatrix}
			(lY)_1&&\\
			&\ddots&\\
			&&(lY)_n
		\end{pmatrix}
		A
		\begin{pmatrix}
			(Xr)_1&&\\
			&\ddots&\\
			&&(Xr)_n
		\end{pmatrix}
		=\lambda^2X
		\begin{pmatrix}
			l_1r_1&&\\
			&\ddots&\\
			&&l_Nr_N
		\end{pmatrix}
		Y.
	\end{equation}
\end{lemma}
\begin{proof}Recall that
	$$f,g:\{1,2,\cdots,N\}\rightarrow\{1,2,\cdots,n\},$$
	and
	$$X_{i,k}=\chi_{\{f(k)=i\}},\quad Y_{k,j}=\chi_{\{g(k)=j\}}.$$
	$$A_{i,j}=\sum_{k=1}^NX_{i,k}Y_{k,j}=\sum_{k:f(k)=i,g(k)=j}1=|f^{-1}(i)\cap g^{-1}(j)|.$$
	Here, $|S|$ denote the number of elements in $S$.
	$$B_{k,m}=\sum_{j=1}^nY_{k,j}X_{j,m}=\sum_{j=1}^n\chi_{\{g(k)=j,f(m)=j\}}=\chi_{\{g(k)=f(m)\}}.$$
	Since $Br=\lambda r$, we have that
	$$\sum_{m=1}^NB_{k,m}r_m=\lambda r_k\Rightarrow \sum_{m=1}^N\chi_{\{g(k)=f(m)\}}r_m=\lambda r_k.$$
	Since $lB=\lambda l$, we have that
	$$\sum_{k=1}^NB_{k,m}l_k=\lambda l_m\Rightarrow \sum_{k=1}^N\chi_{\{g(k)=f(m)\}}l_k=\lambda l_m.$$
	The $(i,j)$ element of the left hand side in \eqref{eq} is
	\begin{align*}
		(lY)_iA_{i,j}(Xr)_j&=|f^{-1}(i)\cap g^{-1}(j)|\left(\sum_{k=1}^Nl_kY_{k,i}\right)\left(\sum_{m=1}^Nr_mX_{j,m}\right)\\
		&=|f^{-1}(i)\cap g^{-1}(j)|\left(\sum_{k=1}^Nl_k\chi_{\{g(k)=i\}}\right)\left(\sum_{m=1}^Nr_m\chi_{\{f(m)=j\}}\right)
	\end{align*}
	Assume that
	$$f^{-1}(i)\cap g^{-1}(j)=\{u_1,u_2,\cdots,u_q\}.$$
	Then
	$$f(u_1)=f(u_2)=\cdots=f(u_q)=i,\quad g(u_1)=g(u_2)=\cdots=g(u_q)=j.$$
	\begin{align*}
		(lY)_iA_{i,j}(Xr)_j&=q\left(\sum_{k=1}^Nl_k\chi_{\{g(k)=i\}}\right)\left(\sum_{m=1}^Nr_m\chi_{\{f(m)=j\}}\right)\\
		&=\sum_{t=1}^q\left(\sum_{k=1}^Nl_k\chi_{\{g(k)=f(u_t)\}}\right)\left(\sum_{m=1}^Nr_m\chi_{\{f(m)=g(u_t)\}}\right)\\
		&=\lambda^2\sum_{t=1}^q l_{u_t} r_{u_t}.
	\end{align*}
	The $(i,j)$ element of the right hand side in \eqref{eq} is
	\begin{align*}
		\lambda^2\sum_{k=1}^NX_{i,k}l_kr_kY_{k,j}&=\lambda^2\sum_{k=1}^Nl_kr_k\chi_{\{f(k)=i,g(k)=j\}}=\lambda^2\sum_{k\in f^{-1}(i)\cap g^{-1}(j)}l_kr_k\\
		&=\lambda^2\sum_{t=1}^q l_{u_t} r_{u_t}.
	\end{align*}
	Thus, we complete the proof.
\end{proof}
\bigskip

\begin{proof}[Proof of Theorem \ref{thm2}]
Suppose $V$ is the weight matrix of $A$, then
\[
V=
X
\begin{pmatrix}
	l_1r_1&&\\
	&\ddots&\\
	&&l_Nr_N
\end{pmatrix}
Y.
\]
By the above lemma, we have that
\[
V=\frac{1}{\lambda^2}
\begin{pmatrix}
	(lY)_1&&\\
	&\ddots&\\
	&&(lY)_n
\end{pmatrix}
A
\begin{pmatrix}
	(Xr)_1&&\\
	&\ddots&\\
	&&(Xr)_n
\end{pmatrix}
\]
Set
$$L=\frac{lY}{\sqrt{\lambda}},\quad R=\frac{Xr}{\sqrt{\lambda}}.$$
Then $L,R$ are the left and right eigenvectors of $A$ corresponding to the eigenvalue $\lambda$. And
$$\sum_{j=1}^nL_jR_j=LR=\frac{1}{\lambda}lYXr=lr=1.$$
Hence,
$$v_{i,j}=\frac{1}{\lambda}L_iA_{i,j}R_j.$$

\end{proof}

\end{document}